\documentclass[final]{tlp}

\usepackage{latexsym}
\usepackage{graphicx}
\newcommand{\algoName}[1]{{\tt #1}}
\newcommand{\bigO}{{\mathcal{O}}}

\newcommand{\horline}{\rule[1mm]{4.8in}{0.1mm}}
\newcommand{\aif}{{\bf if} }
\newcommand{\athen}{{\bf then} }

\newcommand{\afor}{{\bf for} }
\newcommand{\awhile}{{\bf while} }

\newcommand{\ALGO}{{\bf algorithm}}

\newtheorem{eg}{Example}
\newtheorem{definition}{Definition}
\newtheorem{property}{Property}
\newtheorem{corollary}{Corollary}
\newtheorem{proposition}{Proposition}
\newtheorem{theorem}{Theorem}

\newcommand {\bfigalgo} [2]
{
  \begin{figure}
  \vspace*{-0.5cm}
  \begin{tabbing}
    123456\=123\=123\=123\=123\=123\=123\=123\=123\=123\= \kill \\
    \horline\\
     {#1} #2 \{ \\
}

\newcommand {\efigalgo}[1]
{
    \horline
 \end{tabbing}
 \vspace*{-0.5cm}
 \caption{#1}
 \end{figure}
}

\def \color#1]#2{}

\submitted{22nd June 2009} 
\revised{20th October 2009} 
\accepted{8th December 2009} 

\begin{document}

\title[Solving Functional Constraints by Variable Substitution]{
Solving Functional Constraints by Variable Substitution
}

\author[Y. Zhang and R.~H.~C. Yap]
{YUANLIN ZHANG \\ 
Department of Computer Science \\
Texas Tech University, \\
Lubbock, TX79409-3104, USA \\
\email{yzhang@cs.ttu.edu} \\ 
\and
ROLAND H.C. YAP \\ 
School of Computing \\
National University of Singapore \\
13 Computing Drive, 117417, Singapore \\
\email{ryap@comp.nus.edu.sg}
}

\maketitle

\begin{abstract}

{\em Functional constraints} and {\em bi-functional constraints} are
an important constraint class in Constraint Programming (CP)
systems, in particular for Constraint Logic Programming (CLP)
systems. CP systems with finite domain constraints usually employ
CSP-based solvers which use local consistency, for example, arc
consistency. We introduce a new approach which is based instead on
{\em variable substitution}.  We obtain efficient algorithms for
reducing systems involving functional and {\em bi-functional
constraints} together with other non-functional constraints. It also
solves globally any CSP where there exists a variable such that any
other variable is reachable from it through a sequence of functional
constraints. Our experiments on random problems show that variable
elimination can significantly improve the efficiency of solving
problems with functional constraints.
\smallskip

\noindent
{\it To appear in Theory and Practice of Logic Programming (TPLP).}
\end{abstract}


\newcommand{\QED}{$\Box$}

\section{Introduction}

Functional constraints are a common class of constraints occurring
in Constraint Satisfaction Problem(s) (CSP)
\cite{SS77,vHDT92,Kir93}.
Roughly speaking, a constraint $c(x,y)$ is functional if the
value of variable $y$ is some function of the value of variable $x$
(see Definition \ref{def:func}).
Functional constraints arise in two ways, they may occur quite naturally
since one may have ``functions'' or ``equations'' in the constraint model.
Functional constraints also occur systematically in
in Constraint Programming (CP) when the system is a
Constraint Logic Programming (CLP) system.

Functional constraints arise naturally in CLP in two ways.
Firstly, the equations which arise
from matching the head of a rule with an atom in the body
are functional constraints.
Secondly, the basic (or primitive) constraints in a particular instance of
a CLP language will often include functional constraints.
Consider, one of the most widely used and successful constraint domains for
CLP, namely, finite domains which we will call CLP(FD).
The basic constraints in a CLP(FD) system,
for example, CHIP \cite{vHDT92}, can express functional constraints.
An example would be the finite domain constraint, $3X + 2Y = 10$,
with finite domain variables for $X$ and $Y$.\footnote{
Note that the examples which involve CLP use uppercase for variables
as per the logic programming convention. In a more general context,
we will use lowercase variables like $x$ and $y$ for variables.
}
Matching the head and body, gives rise to a number of equations,
and in a FD system, the equations are functional constraints.
For example, when matching
$p(Z^2+1)$ with a rule on $p(X)$ where both $X$ and $Z$ are finite
domain variables, a functional constraint $X = Z^2 + 1$ is produced.\footnote{
Notice that this is not a ``linear'' constraint in the arithmetic sense.
}
We remark that in logic programming, the equations are solved by
unification but in the general setting, constraint solving over the
particular domain is required.
Recognizing and exploiting functional constraints
can facilitate the development of more efficient constraint solvers
for CLP systems.

Most work on solving functional constraints follows the approach in CSP
which is based on arc or path consistency \cite{vHDT92,Dav95}. We
remark that, in many papers, ``functional constraints"
are actually what we call bi-functional constraints
(Definition \ref{def:bifunc}), a special case of functional constraints.
In this paper, we propose a
new method --- {\em variable substitution}
--- to process functional constraints. The idea is that if a
constraint is functional on a variable, this variable in another
constraint can be substituted away using the functional constraint
without losing any solution.\footnote{
A preliminary version of this paper appeared in \cite{ZYLM08}.
}

Given a variable, the {\em variable elimination} method substitutes
this variable in {\em all} constraints involving it such that it is
effectively ``eliminated'' from the problem. This idea is applied to
reduce any problem containing non-functional constraints into a
canonical form where some variables can be safely ignored when
solving the problem. We design an efficient algorithm to reduce, in
$\bigO(ed^2)$ where $e$ is the number of constraints and $d$ the
size of the largest domain of the variables, a general binary CSP
containing functional constraints into a canonical form. This
reduction simplifies the problem and makes the functional portion
trivially solvable. When the functional constraints are also
bi-functional, then the algorithm is linear in the size of the CSP.

Many CLP systems with finite domains make use of
constraint propagation algorithms such as arc consistency.
Unlike arc consistency, our elimination method completely solves the
functional portion of the problem, hence the functional constraints
are eliminated and their consequences are incorporated into the reduced
problem.
Our experiments show that the substitution based
``global'' treatment of functional
constraints can significantly speed up propagation based solvers.

In the rest of the paper, background on CSPs and functional
constraints is given in Section \ref{sec:preliminaries}. Variable
substitution for binary functional constraints is introduced and
studied in Section~\ref{sec:binarySubstitution}.
Section~\ref{sec:generalVariableElimination} presents several
results on algorithms for variable elimination in general CSPs
containing functional constraints. Section \ref{sec:exp} presents an
experimental study on the effectiveness of the variable elimination
algorithm and explains why functional elimination leads
to a smaller problem with a reduced search space.
In Section \ref{sec:beyond}, we extend binary functional constraints
to non-binary functional constraints and we discuss substitution for
general problems where there are non-binary constraints which may be
in extensional and intensional form.
Related work is discussed in
Section~\ref{sec:related}, and the paper is concluded in
Section~\ref{sec:conclusion}.

\section{Preliminaries}
\label{sec:preliminaries}
We begin with the basic concepts and notation
used in this paper.

A binary {\em Constraint Satisfaction Problem (CSP)} (N, D, C) consists of
a finite set of
variables $N= \{v_1,\cdots, v_n \}$, a set of domains $D=\{D_1,\cdots, D_n\}$,
where $D_i$ is the domain of variable $i$ , and a set of constraints
each of which is a binary relation between two variables in $N$.

A constraint between two variables $i$ and $j$ is denoted by
$c_{ij}$. Symbols $a$ and $b$ possibly with subscript denote the
values in a domain. A constraint $c_{ij}$ is a set of allowed
tuples. We assume testing whether a tuple belongs to a constraint
takes constant time. For $a \in D_i$ and $b \in D_j$, we use either
$(a,b) \in c_{ij}$ or $c_{ij}(a,b)$ to denote that values $a$ and
$b$ satisfy the constraint $c_{ij}$. For the problems of interest
here, we require that for all $a \in D_i$ and $b \in D_j$, $(a,b)
\in c_{ij}$ if and only if $(b,a) \in c_{ji}$. If there is no
constraint on $i$ and $j$, $c_{ij}$ denotes a universal relation,
i.e., $D_i \times D_j$.

A constraint graph $G = (V, E)$ where $V = N$ and
$E = \{ \{i,j\} ~|~ \exists c_{ij} \in C \}$.
The constraint graph is usually used to describe the topological
structure of a CSP.
A {\em solution} of a constraint satisfaction problem is an
assignment of a value to each variable such that the assignment
satisfies all the constraints in the problem. A CSP is {\em satisfiable}
if it has a solution. The {\em solution space} of a CSP is the
set of all its solutions. Two CSPs are {\em equivalent} if and only
if they have the same solution space.
Throughout this paper, $n$ represents the number of variables,
$d$ the size of the largest domain of the variables,
and $e$ the number of constraints in $C$.

We need two operations on constraints in this paper. One is the
intersection of two constraints (intersection of the sets of tuples)
that constrain the same set of variables.
The other operation is the composition,
denoted by the symbol ``$\circ$'' of two constraints
sharing a variable. The composition of two relations is:
\[ c_{jk} \circ c_{ij} =
   \{ (a,c)~|~ \exists b \in D_j, such~ that~
          (a,b) \in c_{ij} \wedge (b,c) \in c_{jk}  \}.
\]
Composition is a basic operation in our variable substitution
method. Composing $c_{ij}$ and $c_{jk}$ leads to a new constraint on
variables $i$ and $k$.

\begin{eg}
Consider constraints $c_{ij} = \{(a_1, b_1), (a_2, b_2), (a_2, b_3)\}$ and
$c_{jk} = \{(b_1, c_1), (b_2, c_2),$ $ (b_3, c_2)\}$. The composition of
$c_{ij}$ and $c_{jk}$ is a constraint on $i$ and $k$:
$c_{ik} = \{(a_1, c_1),$ $(a_2, c_2) \}$.
\end{eg}

\begin{definition}
\label{def:func}
A constraint $c_{ij}$ is {\em functional} on variable $j$ if for any
$a \in D_i$ there exists at most one $b \in D_j$ such that
$c_{ij}(a,b)$. $c_{ij}$ is {\em functional} on variable $i$ if
$c_{ji}$ is functional on $i$. Given a constraint $c_{ij}$
functional on variable $j$ and a value $a \in D_i$, we assume
throughout the paper that in constant time we can find the value $b
\in D_j$, if there is one, such that $(a,b) \in c_{ij}$.
\end{definition}

A special case of functional constraints are equations. These are
ubiquitous in CLP.
A typical functional constraint in arithmetic is a
binary linear equation like $2x = 5 - 3y$ which is
functional on $x$ and on $y$.
Functional constraints do not need to be linear. For example,
a nonlinear equation $x^2 = y^2$ where $x,y \in 1..10$ is also
functional on both $x$ and $y$.
In scene labeling problems \cite{Kir93},
there are many functional constraints and other
special constraints.

When a constraint $c_{ij}$ is functional on variable $j$, for
simplicity, we say $c_{ij}$ is functional by making use of the fact
that the subscripts of $c_{ij}$ are an ordered pair. When $c_{ij}$
is functional on variable $i$, $c_{ji}$ is said to be functional.
That $c_{ij}$ is functional does not mean $c_{ji}$ is functional. In
this paper, the definition of functional constraints is different
from the one in \cite{ZYJ99,vHDT92} where constraints are functional
on each of its variables, leading to the following notion.

\begin{definition}
\label{def:bifunc}
A constraint $c_{ij}$ is {\em bi-functional} if $c_{ij}$ is
functional on variable $i$ and also on variable $j$.
\end{definition}

A bi-functional constraint is called {\em bijective} in \cite{Dav95}.
For functional constraints, we have the following property on their
composition and intersection: 1) If $c_{ij}$ and $c_{jk}$ are
functional on variables $j$ and $k$ respectively, their composition
remains functional; and 2) The intersection of two functional
constraints remains functional.

\begin{eg}
The constraint $c_{ij} = \{(a_1, b_1), (a_2, b_1),
(a_3, b_2)\}$ is functional, while the constraint $c_{ij} = \{(a_1,
b_3), (a_2, b_1), (a_3, b_2)\}$ is both functional and
bi-functional. An example of a non-functional constraint is
$c_{ij} = \{(a_1, b_1), (a_1, b_2), (a_2, b_1), (a_3, b_2)\}$.
\end{eg}

In the remainder of the paper, rather than writing $v_i$,
we will simply refer to a variable by its subscript,
i.e. $i$ rather than $v_i$.

\section{Variable Substitution and Elimination Using
Binary Functional Constraints}
\label{sec:binarySubstitution}

We introduce the idea of variable substitution. Given a CSP $(N, D,
C)$, a constraint $c_{ij} \in C$ that is functional on $j$, and a
constraint $c_{jk}$ in $C$, we can substitute $j$ by $i$ in
$c_{jk}$ by composing $c_{ij}$ and $c_{jk}$. If there is already a
constraint $c_{ik} \in C$, the new constraint on $i$ and $k$ is
simply the intersection of $c_{ik}$ and $c_{jk} \circ c_{ij}$.

\begin{definition} \label{def:substitution}
Consider a CSP $(N,D,C)$, a constraint $c_{ij} \in C$ functional
on $j$, and a constraint $c_{jk} \in C$. To {\em substitute $j$ by $i$}
in $c_{jk}$, using $c_{ij}$, is to get a new CSP where $c_{jk}$ is replaced by
$c'_{ik} = c_{ik} \cap (c_{jk} \circ c_{ij})$.
The variable $i$ is called the {\em substitution variable}.
\end{definition}

A fundamental property
of variable substitution is that it preserves the solution space of
the problem.

\begin{property} \label{prop:substitution}
Given a CSP $(N,D,C)$, a constraint $c_{ij} \in C$ functional on $j$, and
a constraint $c_{jk} \in C$, the new problem obtained by substituting $j$
by $i$ in $c_{jk}$ is equivalent to $(N,D,C)$.
\end{property}

\begin{proof}
Let the new problem after substituting $j$ by $i$ in
$c_{jk}$ be $(N, D, C')$ where $C' = (C - \{c_{jk}\}) \cup \{c'_{ik}\}$
and $c'_{ik} = c_{ik} \cap (c_{jk} \circ c_{ij})$.

Assume ($a_1, a_2, \cdots, a_n$) is a solution of ($N, D, C$). We
need to show that it satisfies $C'$. The major difference between
$C'$ and $C$ is that $C'$ has new constraint $c'_{ik}$. It is known
that $(a_i, a_j) \in c_{ij}$, $(a_j, a_k) \in c_{jk}$, and  if there
is $c_{ik}$ in $C$, $(a_i, a_k) \in c_{ik}$. The fact that $c_{ik}'
= (c_{jk} \circ c_{ij}) \cap c_{ik}$ implies $(a_i, a_k) \in
c'_{ik}$. Hence, $c'_{ik}$ is satisfied by ($a_1, a_2, \cdots,
a_n$).

Conversely, we need to show that any solution ($a_1, a_2, \cdots, a_n$) of
($N, D, C'$) is a solution of ($N, D, C$). Given the difference between
$C'$ and $C$, it is sufficient to show the solution satisfies
$c_{jk}$.
We have $(a_i, a_j) \in c_{ij}$ and $(a_i, a_k) \in c'_{ik}$.
Since $c'_{ik} = (c_{jk} \circ c_{ij}) \cap c_{ik}$,
there must exist $b \in D_{j}$ such that $(a_i, b) \in c_{ij}$
and $(b, a_k) \in c_{jk}$.
As $c_{ij}$ is functional, $b$ has to be $a_j$.
Hence, $a_j$ and $a_k$ satisfy $c_{jk}$.
\end{proof}

Based on variable substitution, we can eliminate a variable from
a problem so that no constraint will be on this variable
(except the functional constraint used to substitute it).

\begin{definition} \label{def:elimination}
Given a CSP $(N, D, C)$ and a constraint $c_{ij} \in C$ functional on $j$,
to {\em eliminate $j$ using $c_{ij}$} is to substitute $j$ by $i$,
using $c_{ij}$, in every constraint $c_{jk} \in C$
(except $c_{ji}$).
\end{definition}

We can also substitute $j$ by $i$ in  $c_{ji}$ to obtain $c'_{ii}$ and then
intersect $c'_{ii}$ with the identity relation on $D_i$, equivalent to
a direct revision of the domain of $i$ with respect to $c_{ij}$.
This would make the algorithms presented in this paper more
uniform, i.e., only operations on constraints are used.
Since in most algorithms we want to make domain revision explicit,
we choose not to substitute $j$ by $i$ in $c_{ji}$.

Given a functional constraint $c_{ij}$ of a CSP $(N,D,C)$,
let $C_j$ be the set of all constraints involving $j$, except $c_{ij}$.
The elimination of $j$ using $c_{ij}$ results in a new problem $(N,D,C')$ where
\[C' = (C - C_j) \cup \{c'_{ik} ~|~
  c'_{ik} = (c_{jk} \circ c_{ij}) \cap c_{ik}, c_{jk} \in C\}.\]
In the new problem, there is only one constraint $c_{ij}$ on $j$ and thus
$j$ can be regarded as being ``eliminated''.

\begin{eg}
Consider a problem with three constraints whose
constraint graph is shown in Figure \ref{fig:elimination}(a). Let
$c_{ij}$ be functional which this is indicated by the arrow in the diagram.
The CSP after $j$ has been eliminated using
$c_{ij}$ is shown in Figure~\ref{fig:elimination}(b). In the new
CSP, constraints $c_{jk}$ and $c_{jl}$ are discarded, and new
constraints $c_{ik} = c_{jk} \circ c_{ij}$ and $c_{il} = c_{jl}
\circ c_{ij}$ are added.
Note that the other edges are not directed as the
constraints $c_{jk}, c_{jl}, c_{ik}, c_{il}$ may not be functional.
\end{eg}

\begin{figure}
\includegraphics[width=9cm]{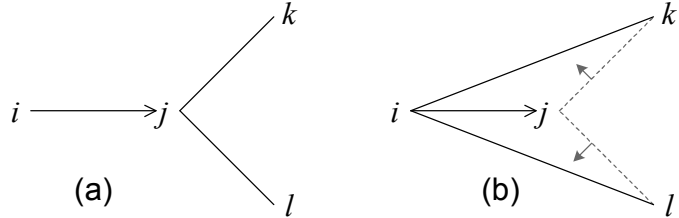}
\caption{(a): A CSP with a functional constraint $c_{ij}$. (b): The
new CSP after eliminating the variable $j$ using $c_{ij}$.}
\label{fig:elimination}
\end{figure}

The variable elimination involves ``several" substitutions and thus
preserves the solution space of the original problem by
Property~\ref{prop:substitution}.

\begin{corollary} \label {prop:elimination}
Given a CSP ($N, D, C$) and a functional constraint $c_{ij} \in C$,
the new problem ($N, D, C'$) obtained by the elimination of variable $j$
using $c_{ij}$ is equivalent to ($N, D, C$).
\end{corollary}

\section{Elimination Algorithms for CSPs with Functional Constraints
and Non-Functional Constraints}
\label{sec:generalVariableElimination}
We now extend variable elimination to general CSPs with functional
and non-functional constraints.
The idea of variable elimination (Definition~\ref{def:elimination} in
Section~\ref{sec:binarySubstitution}) can be used to reduce
a CSP to the following canonical functional form.

\begin{definition}
A CSP $(N,D,C)$ is in {\em canonical functional form}
if for any constraint $c_{ij} \in C$ functional on $j$,
the following conditions are satisfied:
1) if $c_{ji}$ is also functional on $i$(i.e., $c_{ij}$ is bi-functional),
either $i$ or $j$ is not constrained by any other constraint in $C$;
2) otherwise, $j$ is not constrained by any other constraint in $C$.
\end{definition}

As a trivial example, a CSP without any functional constraint
is in canonical functional form. If a CSP contains some
functional constraints, it is in canonical functional form
intuitively if for any functional constraint $c_{ij}$,
there is only one constraint on $j$. As an exception, the first
condition in the definition implies that
when $c_{ij}$ is bi-functional, one variable of $\{i, j\}$
might have several bi-functional constraints on it.

In a canonical functional form CSP,
the functional constraints form {\em disjoint star graphs}.
A {\em star graph} is a tree where there exists a node,
called the {\em center}, which we call the
{\em free variable}, such that there is an edge between
this center node and every other node, which we call
and eliminated variable.
The constraint between the free variable and eliminated variable is that
it is functional on the eliminated variable.
In Figure~\ref{fig:elimination}(a), assuming $c_{jk}$ and $c_{jl}$
are functional on $k$ and $l$ respectively, then there would be
directed edges (arrows) from $j$ to $k$ and $j$ to $l$. After
eliminating $j$, we get a star graph in
Figure~\ref{fig:elimination}(b), since $i$ will be the free variable
at the center of the star graph, with free variables $k$ and $l$.
Notice that before eliminating $j$, Figure~\ref{fig:elimination}(a)
is a star graph, but the constraints are not in a canonical form.

The constraint between a free variable $i$
and an eliminated variable $j$ is functional on $j$, but it may or
may not be functional on $i$. In the special case that the star
graph contains only two variables $i$ and $j$ and $c_{ij}$ is
bi-functional, one of the variables can be called a free
variable while the other is called an eliminated variable.

If a CSP is in canonical functional form, all functional constraints
and the eliminated variables can be {\em ignored} when we try to find a solution
for this problem.
Thus, to solve a CSP $(N,D,C)$ in canonical
functional form whose non-eliminated variables are $NE$, we only need to
solve a smaller problem $(NE, D', C')$
where $D'$ is the set of domains of the variables $NE$ and
$C' = \{c_{ij} ~|~ c_{ij} \in C \mbox{ and } i,j \in NE \}$.

\begin{proposition} \label{prop:solveCanonicalForm}
Consider a CSP $P_1=(N,D,C)$ in a canonical functional form
and a new CSP $P_2=(NE, D', C')$ formed by ignoring the eliminated
variables in $P_1$. For any free variable $i \in N$ and any
constraint $c_{ij} \in C$ functional on $j$, assume any value
of $D_i$ has a support in $D_j$ and this
support can be found in constant time.
Any solution of $P_2$ is extensible to a unique solution of
$P_1$ in $\bigO(|N-NE|)$ time.
Any solution of $P_1$ can be
obtained from a solution of $P_2$.
\end{proposition}

\begin{proof}
Let $(a_1, a_2, \cdots, a_{|NE|})$ be a solution of $(NE, D', C')$.
Consider any eliminated variable $j \in N-NE$. In $C$, there is only
one constraint on $j$.
Let it be $c_{ij}$ where $i$ must be a free
variable. By the assumption of the proposition, the value
of $i$ in the solution has a unique support in $j$.
This support will be assigned to $j$.
In this way, a unique solution for $(N,D,C)$
is obtained. The complexity of this extension is $\bigO(|N-NE|)$.

Let $S$ be a solution of $(N,D,C)$ and $S'$ the portion of $S$
restricted to the variables in $NE$.
$S'$ is a solution of $(NE, D', C')$ because
$C' \subseteq C$.
$S'$ can then be extended to $S$ by using the functional constraints
on the values of the free variables in $S'$
to give unique values for the variables in $N-NE$.
\end{proof}

Any CSP with functional constraints can be transformed into
canonical functional form by variable elimination using the
algorithm in Figure~\ref{algo:variable-elimination}. Given a
constraint $c_{ij}$ functional on $j$, line 1 of the algorithm
substitutes $j$ by $i$ in all constraints involving $j$. Note the
arc consistency on $c_{ik}$, for all neighbor $k$ of $i$, is
enforced by line 3.

\bfigalgo {\ALGO} {Variable\_Elimination({\bf inout} $(N, D, C)$,
    {\bf out} {\em consistent}) }

  \> $L \leftarrow N$; \\

  \> \awhile(There is $c_{ij} \in C$ functional on $j$ where $i,j \in L$ and $i \neq j$)\{ \\
  \>\>  // Eliminate variable $j$, \\
1.  \>\>  $C \leftarrow \{c'_{ik} ~|~
     c'_{ik} \leftarrow (c_{jk} \circ c_{ij}) \cap c_{ik}, c_{jk} \in C, k \neq i\} \cup (C - \{c_{jk} \in C ~|~ k \neq i\})$; \\
2.  \>\>  $L \leftarrow L -\{ j \}$; \\
3.  \>\>  Revise the domain of $i$ wrt $c_{ik}$ for every neighbour $k$ of $i$; \\
  \>\>  \aif ($D_i$ is empty) \athen \{ $consistent$ $\leftarrow$ {\bf false}; return \} \\
  \> \}\\
  \> $consistent$ $\leftarrow$ {\bf true}; \\
\}\\
\efigalgo {\label {algo:variable-elimination} A variable elimination algorithm
to transform a CSP into a canonical functional form.}

\begin{theorem} \label{th:basic}
Given a CSP $(N,D,C)$, \algoName{Variable\_Elimination} transforms
the problem into a canonical functional form in $\bigO(n^2d^2)$.
\end{theorem}

\begin{proof}
Assume \algoName{Variable\_Elimination} transforms a CSP
$P_1=(N, D, C)$ into a new problem $P_2=(N,D',C')$. We show that
$P_2$ is in canonical functional form. For any constraint $c_{ij}
\in C'$ functional on $j$, there are two cases. Case 1: $j \notin L$
when the algorithm terminates. Since $j \in L$ when the algorithm
starts and line 2 is the only place where $j$ can be removed from
$L$, $j$ must have been eliminated at certain step of the while
loop. In line 1 (the component after ``$\cup$''), all constraints on
$j$ (except $c_{ij}$) are removed. That is $c_{ij}$ is the unique
constraint on $j$. Case 2: $j \in L$ when the algorithm terminates.
Since $c_{ij}$ is functional on $j$, variable $i$ is not in $L$ when
the algorithm terminates (otherwise, $j$ will be substituted by line
1 at certain step of the while loop). Therefore, $i$ is removed from
$L$ at certain step of the while loop. $i$ is not substituted using
$c_{ki}$ where $k \neq j$ (otherwise $c_{ij} \not \in C'$ because of
the elimination of $i$). This implies that $i$ was substituted using
$c_{ji}$, and thus $c_{ji}$ is functional on $i$ by the loop
condition. Hence, $c_{ij}$ is bi-functional, and $i$ is not
constrained by any other constraints (thanks to line 1). Therefore,
cases 1 and 2 show that $P_2$ is in canonical functional form.

Next, we prove the complexity of \algoName{Variable\_Elimination}.
The algorithm eliminates any variable in $N$ at most once because
once it is eliminated it is removed from $L$ (line 2). Assume,
before the algorithm, there is at most one constraint on any pair of
variables (otherwise, we take the intersection of all constraints on
the same pair of variables as the unique constraint). This property
holds during the elimination process because in line 1, the
intersection in the component before ``$\cup$'' guarantees that we
have only one copy of constraint on any two variables. So, for each
variable $j$ and a constraint $c_{ij}$ functional on $j$, there are
at most $n-2$ other constraints on $j$. The variable $j$ in those
constraints needs to be substituted (line 1).

The complexity of the substitution $j$ in each constraint is
$\bigO(d^2)$ which is the cost of the composition of a functional
constraint and a general constraint. Recall that, for a functional
constraint $c_{ij}$, given a value $a \in D_i$, we can find its
support in $D_j$ in constant time. To compose $c_{ij}$ with a
general constraint $c_{jk}$, for each value $a \in D_i$, we find its
support $b \in D_j$ (in constant time). If we take each constraint
as a matrix, the row of $b$ of $c_{jk}$ will be the row of $c_{jk}
\circ c_{ij}$, which takes $d$ steps. Therefore, the cost of
computing $c_{jk} \circ c_{ij}$ is $\bigO(d^2)$.

For $n-2$ constraints, the elimination of $j$ (line 1) takes
$\bigO(nd^2)$. There are at most $n-1$ variables to eliminate and
thus the worst case complexity of the algorithm is $\bigO(n^2d^2)$.
\end{proof}

It is worth noting that the variable elimination algorithm
is able to globally solve some CSPs containing non-functional constraints.

\begin{eg}
Consider a simple example where there are three variables $i,j$, and
$k$ whose domains are $\{1,2,3\}$ and the constraints are $i=j$,
$i=k+1$, and $j\neq k$. Note that although the constraints are
listed in an equational form, the actual constraints are explicit
and discrete, thus normal equational reasoning might not be
applicable. By eliminating $j$ using $c_{ij}$, $c_{ik}$ becomes
$\{(2,1), (3,2)\}$, and the domain of $i$ becomes $\{2,3\}$. The
non-functional constraint $c_{jk}$ is gone. The problem is in
canonical functional form. A solution can be obtained by letting $i$
be $2$ and consequently $j=2$ and $k=1$.
\end{eg}

By carefully choosing an ordering of the variables to eliminate, a
faster algorithm can be obtained. The intuition is that once a
variable $i$ is used to substitute for other variables, $i$ itself
should not be substituted by any other variable later.

\begin{eg}
Consider a CSP with functional constraints $c_{ij}$
and $c_{jk}$. Its constraint graph is shown in
Figure~\ref{fig:funElimOrder}(a) where a functional constraint is
represented by an arrow. If we eliminate $k$ and then $j$, we first
get $c_{jl_1}$ and $c_{jl_2}$, and then get $c_{il_1}$ and
$c_{il_2}$. Note that $k$ is first substituted by $j$ and
then later $j$ is substituted by $i$.
If we eliminate $j$ and then $k$, we
first get $c_{ik}$, and then get $c_{il_1}$ and $c_{il_2}$. In this
way, we reduce the number of compositions of constraints.
\end{eg}

\begin{figure}
\includegraphics{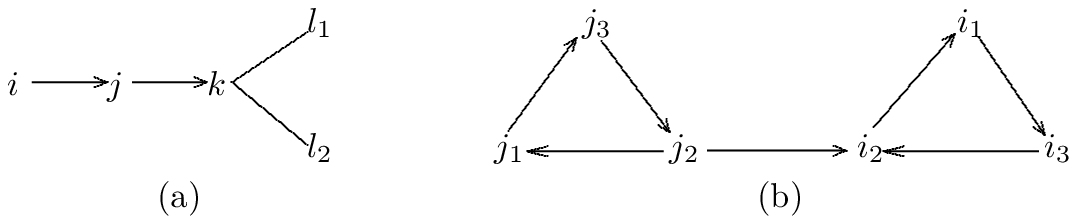}
\caption{\label{fig:funElimOrder} (a) The constraint graph of a CSP
with functional constraints $c_{ij}$ and $c_{jk}$. (b) A directed
graph.}
\end{figure}

Given a CSP $P=(N,D,C)$, $P^F$ is used to denote its directed graph
$(V, E)$ where $V = N$ and $E = \{ (i, j) ~|~ c_{ij} \in C \mbox{
and $c_{ij}$ is functional on $j$} \}$. Non-functional constraints
in $C$ do not appear in $P^F$. A subgraph of a directed graph is
{\em strongly connected} if for any two vertices of the subgraph,
any one of them is reachable from the other. A {\em strongly
connected component} of a directed graph is a maximum subgraph that
is strongly connected. To describe our algorithm we need the
following notation.

\begin{definition}
Given a directed graph $(V,E)$, a sequence of the nodes of $V$ is a
\emph{functional elimination ordering} if for any two nodes $i$ and
$j$, $i$ before $j$ in the sequence implies that there is a path
from $i$ and $j$. A \emph{functional elimination ordering of a CSP
problem} $P$ is a functional elimination ordering of $P^F$.
\end{definition}

The functional elimination ordering is used to overcome the
redundant computation shown in the example on
Figure~\ref{fig:funElimOrder}(a). Given a directed graph $G$, a
functional elimination ordering can be found by:
1) finding all the strongly connected components of $G$;
2) modifying $G$ by taking every component as one vertex with
  edges changed and/or added accordingly;
3) finding a topological ordering of the nodes in the new
  graph; and
4) replacing any vertex $v$ in the ordering by any sequence of the vertices of the
  strongly connected component represented by $v$.

To illustrate the process, consider the example in
Figure~\ref{fig:funElimOrder}(b) which can be taken as $P^F$ for some
CSP problem $P$. All strongly connected components are $\{j_1, j_2,
j_3\}$, denoted by $c_1$, and $\{i_1, i_2, i_3\}$, denoted by $c_2$.
We construct the new graph by replacing the components by vertices:
$(\{c_1, c_2\}, \{(c_1, c_2)\})$. We have the edge $(c_1, c_2)$
because the two components are connected by $(j_2, i_2)$. The
topological ordering of the new graph is $\langle c_1, c_2 \rangle$.
Now we can replace $c_1$ by any sequence of $j$'s and $c_2$ by any
sequence of $i$'s. For example, we can have a functional elimination
ordering $\langle j_3, j_2, j_1, i_2, i_3, i_1 \rangle$.

The algorithm \algoName{Linear\_Elimination} in
Figure~\ref{algo:linearElimination} first finds a functional
elimination ordering $O$ (line 1). 
The body of the while loop at line 4 is
to \emph{process} all the variables in $O$. Every variable $i$ of
$O$ is \emph{processed} as follows: $i$ will be used to substitute
for all the variables \emph{reachable} from $i$ \emph{through
constraints that are functional in $C^0$ and still exist in the
current $C$}. Those constraints are called \emph{qualified}
constraints. Specifically, $L$ initially holds the immediate
reachable variables through qualified constraints (line 8). Line 9
is a loop to eliminate all variables reachable from $i$. The loop at
line 11 is to eliminate $j$ using $i$ from the current $C$. In this
loop, if a constraint $c_{jk}$ is qualified (line 14), $k$ is
reachable from $i$ through qualified constraints. Therefore, it is
put into $L$ (line 15).

\newcommand{\done}{\emph{eliminated}}

\bfigalgo {\ALGO} {Linear\_Elimination({\bf inout} $(N, D, C))$}
\ 1. Find a functional elimination ordering $O$ of the problem; \\
\ 2. Let $C^0$ be $C$; any $c_{ij}$ in $C^0$ is denoted by $c_{ij}^0$;\\
\ 3. For each $i \in N$, it is marked as {\em not} \done; \\
\ 4. \awhile($O$ is not empty) \{ \\
\ 5. \> Take and delete the first variable $i$ from $O$; \\
\ 6. \> \aif ($i$ is {\bf not} \done) \{ \\
\ 8. \> \> $L \leftarrow \{j ~|~ (i,j) \in C \mbox{ and } c_{ij}^0 \mbox{ is functional} \}$; \\
\ 9. \> \> \awhile($L$ not empty) \{\\
\    \> \> \> Take and delete $j$ from $L$; \\
\ 11. \> \> \> \afor any $c_{jk} \in C-\{c_{ji}\}$ \{ // Substitute $j$ by $i$ in $c_{jk}$; \\
\    \> \> \> \> $c'_{ik} \leftarrow c_{jk} \circ c_{ij} \cap c_{ik}$; \\
\    \> \> \> \> $C \leftarrow C \cup \{c'_{ik}\} - \{c_{jk}\}$; \\
\ 14.   \> \> \> \> \aif ($c_{jk}^0$ is functional) \athen\\
\ 15.   \> \> \> \> \> $L \leftarrow L \cup \{k\}$;\\
\    \> \> \> \} \\
\ 16.   \> \> \> Mark $j$ as \done; \\
\    \> \> \} // loop on $L$ \\
\ \> \} \\
\ \ \ \ \} // loop on $O$ \\
\ \} // end of algorithm \\
\efigalgo {\label {algo:linearElimination} A variable elimination
algorithm of complexity $O(ed^2)$.}

To illustrate the ideas underlying the algorithm, consider the
example in Figure~\ref{fig:funElimOrder}(b). Now, we assume the edges
in the graph are the only constraints in the problem. Assume the
algorithm finds the ordering given earlier: $O = \langle j_3, j_2,
j_1, i_2, i_3, i_1 \rangle$. Next, it starts from $j_3$. The
qualified constraints leaving $j_3$ are $c_{j_3j_2}$ only. So, the
immediate reachable variables through qualified constraints are $L =
\{j_2\}$. Take and delete $j_2$ from $L$.
Substitute $j_2$ by $j_3$
in constraints $c_{j_2i_2}$ and $c_{j_2j_1}$. As a result,
constraints $c_{j_2i_2}$ and $c_{j_2j_1}$ are removed from $C$ while
$c_{j_3j_1} = c_{j_3j_1} \cap (c_{j_2j_1} \circ c_{j_3j_2})$ and new
constraint $c_{j_3i_2} = c_{j_2i_2} \circ c_{j_3j_2}$ is introduced
to $C$. One can verify that both $c_{j_2j_1}$ and $c_{j_2i_2}$ are
qualified. Hence, variables $j_1$ and $i_2$ are reachable from $j_3$
and thus are put into $L$. Assume $j_1$ is selected from $L$. Since
there are no other constraints on $j_1$, nothing is done. Variable
$i_2$ is then selected from $L$. By eliminating $i_2$ using $j_3$,
$c_{i_2i_1}$ and $c_{i_2i_3}$ are removed from $C$ and $c_{j_3i_1}$
and $c_{j_3i_3}$ are added to $C$. Constraint $c_{i_2i_1}$ is
qualified, and thus $i_1$ is added to $L$. Note that $c_{i_2i_3}$ is
not qualified because it is not functional on $i_3$ in terms of the
graph. We take out the only variable $i_1$ in $L$. After $i_1$ is
eliminated using $j_3$, $c_{i_1i_3}$ is removed from $C$, and
constraint $c_{j_3i_3}$ is updated to be $c_{j_3i_3} \cap
(c_{i_1i_3} \circ c_{j_3i_1})$. Since $c_{i_1i_3}$ is qualified,
$i_3$ is added to $L$. One can see that although $i_3$ was not
reachable when $i_2$ was eliminated, it finally becomes reachable
because of $i_1$.
All the variables in a strongly connected component are reachable
from the variable under processing if one of them is reachable. Now,
take $i_3$ out of $L$, and nothing is done because there are no
other constraints incident on it. Every variable except $j_3$ is
marked as \done\ (line 16), the {\bf while} loop on $O$ (line 4 and
6) terminates.

\begin{theorem} \label{th:main}
Given a CSP problem, the worst case time complexity of \\
\algoName{Linear\_Elimination} is $O(ed^2)$ where $e$ is the number
of constraints and $d$ the size of the maximum domain in the problem.
\end{theorem}

\begin{proof}
To find a functional elimination ordering involves the
identification of strongly connected components and topological
sorting. Each of the two operations takes linear time. Therefore,
line 1 of the algorithm takes $\bigO(n+e)$.

The {\bf while} loop of line 4 takes $\bigO(ed^2)$. Assume that
there is a unique identification number associated with each
constraint in $C$. After some variable of a constraint is
substituted, the constraint's identification number refers to the
new constraint.

For any identification number $\alpha$, let its first associated
constraint be $c_{jk}$. Assuming $j$ is substituted by some other
variable $i$, we can show that $i$ will be never be substituted
later in the algorithm. By the algorithm, $i$ is selected at line 5.
Since it is the first element of $O$ now, all variables before $i$
in the original functional ordering have been processed. Since $i$
is not eliminated, it is not reachable from any variable before it
(in terms of the original $O$) through qualified constraints (due to
the loop of line 9). Hence, there are two cases: 1) there is no
constraint $c_{mi}$ of $C$ such that $c_{mi}^0$ is functional on
$i$, 2) there is at least one constraint $c_{mi}$ of $C$ such that
$c_{mi}^0$ is functional on $i$. In the first case, our algorithm
will never substitute $i$ by any other variable. By definition of
functional elimination ordering, case 2 implies that $i$ belongs to
a strongly connected component whose variables have not been
eliminated yet. Since all variables in the component will be
substituted by $i$, after the loop of line 9, there is no constraint
$c_{mi}$ of $C$ such that $c_{mi}^0$ is functional on $i$. Hence,
$i$ will never be substituted again.

In a similar fashion, if variable $k$ is substituted by $l$, $l$
will never be substituted later by the algorithm.

So, there are at most two substitutions occurring to $\alpha$. Each
of these substitutions is a composition that involves a functional
constraint. Hence, its complexity is $O(d^2)$ in the worst case as
shown in the proof of Theorem~\ref{th:basic}.

Since there is a unique identification number for each constraint,
the total number of the unique identification numbers is $e$ and
thus the time taken by the {\bf while} loop at line 4 is $O(ed^2)$.
In summary, the worst case time complexity of the algorithm is
$O(ed^2)$. 
\end{proof}

Before proving some properties of \algoName{Linear\_Elimination}, we
first define trivially functional constraints.

\begin{definition}
Given a problem $P$, let $C^0$ be the constraints before applying
\algoName{Linear\_Elimination} and $C$ the constraints of the problem
at any moment during the algorithm. A constraint $c_{ij}$ of $C$ is
{\em trivially functional} if it is functional and satisfies the
condition: $c_{ij}^0$ is functional or there is a path $i_1 (=i),
i_2, \cdots, i_m (=j)$ in $C_0$ such that, $\forall k \in 1..m-1$,
$c_{i_k i_{k+1}}^0$ is functional on $i_{k+1}$.
\end{definition}

\begin{theorem}
Algorithm \algoName{Linear\_Elimination} transforms a CSP $(N,D,C)$
into a canonical functional form if all newly produced functional
constraints (due to substitution) are trivially functional.
\end{theorem}
The proof of this result is straightforward and thus omitted here.

\begin{corollary}
For a CSP problem with non-functional constraints and bi-func\-tional
constraints, the worst case time complexity of algorithm
\algoName{Linear\_Elimination} is linear to the problem size.
\end{corollary}
This result follows the observation below. When the functional
constraint involved in a substitution is bi-functional, the
complexity of the composition is linear to the constraints involved.
From the proof of Theorem~\ref{th:main}, the complexity of the
algorithm is linear to the size of all constraints, i.e., the
problem size.

\begin{corollary}
Consider a CSP with both functional and non-functional constraints.
If there is a variable of the problem such that every variable of
the CSP is reachable from it in $P^F$, the satisfiability of the
problem can be decided in $\bigO(ed^2)$ using
\algoName{Linear\_Elim\-ination}.
\end{corollary}

For a problem with the property given in the corollary, its
canonical functional form becomes a star graph. So, any value in the
domain of the free variable is extensible to a solution if we add
(arc) consistency enforcing during \algoName{Linear\_Elimination}.
The problem is not satisfiable if a domain becomes empty during the
elimination process.

\section{Experimental Results}
\label{sec:exp}

We experiment to investigate the effectiveness of variable elimination on
problem solving. In our experiments, a problem is solved in two
ways: (i) directly by a general solver; and (ii) variable elimination is
applied before the solver.

There are no publicly available benchmarks on functional
constraints. We test the algorithms on random problems which are
sufficiently hard so that we can investigate the effect of different numbers
of functional constraints and the effect of constraint tightness.

We generate random problems $\langle n, d, e, nf, t \rangle$ where
$n$ is the number of variables, $d$ domain size, $e$ the number of
constraints, {\em nf} the number of functional constraints, and $t$
the tightness of non-functional constraints. The tightness $r$ is
defined as the percentage of allowed tuples over $d^2$.
There are {\em nf} functional constraints and the rest
of the binary constraints are non-functional.
Each functional constraint is constructed to have $d$
allowed tuples.
In the context of random problems, the tightness factor of $1/d$ due
to the functional constraints is rather tight. When we increase {\em
nf}, it can be the case that the search space is quickly reduced due
to the effect of these very tight constraints. Therefore the
``hardness'' of the problems drops correspondingly when there is a
significant increase of $nf$. As described below, we try to counter
this effect by removing problems which be solved too easily from the
benchmarks.

In the experiments, we systematically test benchmark problems
generated using the following parameters: $n,d$ are $50$, $e$ varies
from $100$ to $710$ with step size $122$ (710 is $\sim 10\%$ of the
total number of possible constraints (1225)), {\em nf} varies from
$2$ to $12$, and $t$ varies from $0.2$ to $1.0$ with step size
$0.05$. When {\em nf} is small (for example, 2), there are so many hard
problems that we can only experiment with a small portion of the
problems because it is computationally infeasible. When {\em nf} is
large (for example, 12), even the most difficult problem instances from the
set of instances becomes easy and only a small number of backtracks
is needed. These instances can be solved too easily and thus are not
very useful for benchmarking the elimination algorithm (i.e., we
will not expect the elimination algorithm to bring any benefits to
these instances). So, we do not include the cases with $nf > 12$.
When $nf=12$, the most difficult problems we found are with $e=710$.
Table \ref{tab:tightness} shows the hardness of problems instances,
with $nf=12$ and $e=710$, in terms of the number of backtracks \#bt
(average of 10 instances) needed. The hardness is measured using an
arc consistency solver without using the elimination algorithm. When
$t$ is from $0.2$ to $0.65$, the problems are too easy (\#bt is 0).
For the most difficult case of $t$ being $0.8$, \#bt is still rather
small (around $1000$). On the other hand, when {\em nf} is small,
one can expect that the application of elimination may not make much
difference. Therefore, we do not include the cases when {\em nf} is
small either.

\begin{table}
      \begin{tabular}{rrrrrrr}
        \hline\hline
        tightness ($t$) & 0.2 -- 0.65 & 0.7 & 0.75& 0.8 & 0.85 & 0.9 -- 0.95\\
        \hline
        \#bt & 0 & 5.7 & 22.9 & 1023 & 0.2 & 0 \\
        \hline\hline
      \end{tabular}
\caption{Hardness of the problem versus tightness}
\label{tab:tightness}
\end{table}

Due to the observations above, we evaluate the algorithm only on
non-trivial problem instances and where $nf$ is not too tiny.
For each {\em nf} (varying from $6$ to $12$), the results of the most
difficult problem instances discovered in the exploration process
above is shown in Table~\ref{tbl:random}. The results were obtained
on a DELL PowerEdge 1850 (two 3.6GHz Intel Xeon CPUs) with Linux. We
implement both the elimination algorithm and a general solver in C++.
The solver uses the standard backtracking algorithm armed with
arc consistency enforcing algorithm after each variable assignment.
During the search, the dom/deg heuristic is used to select a variable,
and the value selection heuristic is in lexicographical order.\footnote{
The dom/deg heuristic is a dynamic variable selection heuristic.
}

\begin{table}
\caption{
The experimental results for random problems with $n=d=50$.}
\begin{center}
\begin{tabular}{rrrrrrr} \hline\hline
 & &  & \multicolumn{2}{c}{Elimination} &
\multicolumn{2}{c}{No Elimination} \\ 
\raisebox{1.5ex}[0pt]{$e$} & \raisebox{1.5ex}[0pt]{{\em nf}} &
\raisebox{1.5ex}[0pt]{tightness} & cpu time (s) & \#backtracks & cpu
time (s) & \#backtracks \\ \hline
344 & 6 & 0.60 & 20042 & 1.889e+06 & 47781 & 5.381e+06 \\ 
466 & 7 & 0.70 & 9266 & 9.362e+05 & 35136 & 3.955e+06 \\ 
588 & 8 & 0.75 & 17922 & 1.635e+06 & 45464 & 4.386e+06 \\ 
588 & 9 & 0.75 & 10346 & 5.679e+05 & 21231 & 1.605e+06 \\ 
710 & 10 & 0.80 & 3039 & 2.244e+05 & 5771 & 5.146e+05 \\ 
710 & 11 & 0.80 & 481 & 26522 & 959.7 & 71258 \\ 
588 & 12 & 0.75 & 24 & 682 & 57.9 & 2960 \\ \hline\hline
\end{tabular}
\end{center}
\label{tbl:random}
\end{table}

In Table~\ref{tbl:random}, the cpu time is the total time of twenty
problem instances for a given combination of $e$, $nf$ and tightness,
and the number of backtracks are their average.
For the problem instances used in
Table~\ref{tbl:random}, the time to transform the instances into
their canonical forms is negligible compared to the time needed for
solving the instance. There are several reasons. First, the number of
constraints involved in the elimination is relatively small compared
to the total number of constraints in the problems. Second, the
algorithm is as efficient as the optimal general arc consistency
algorithm used in the solver.
Thirdly, the elimination is applied only once to reduce the problem
which can be done before the backtracking search, while the arc consistency
algorithm needs to be called at every step during the search, i.e.
roughly about the same as the number of backtracks.

The results show that the variable elimination can significantly
speed up the problem solving in terms of both cpu time and the
number of backtracks. It reduces the number of backtracks by two to
four times and also reduces the cpu time correspondingly.

The statistics (cpu time and number of backtracks) used in
Table~\ref{tbl:random} is for 20 problem instances
for each value of the selected parameters (each row in the table).
We notice that the hardness of these instances is not uniform,
i.e., some instances are significantly harder than the others.
To better visualize the performance of the algorithms,
we replot the same data from Table~\ref{tbl:random}
in Figure~\ref{fig:allinstances}.
Each data point in Figure~\ref{fig:allinstances},
represents an instance whose
$x$-coordinate is the number of backtracks with elimination applied
while its $y$-coordinate is that without elimination.
Both axis use a log scale.

\begin{figure}
\includegraphics{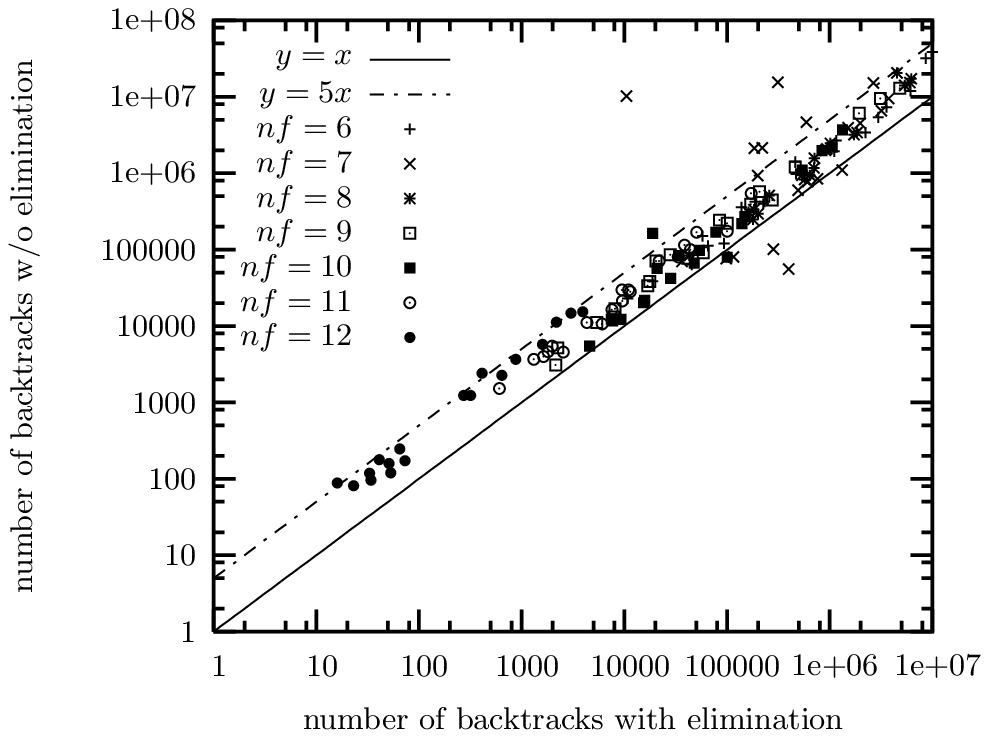}
    \caption{\label{fig:allinstances} The number of backtracks
    needed with and without elimination.}
\end{figure}

The scatter plot in Figure~\ref{fig:allinstances} shows a similar
performance improvement resulting from elimination.
Note that all points above the line $y = x$ indicate that not using
elimination requires more backtracks.
We highlight the instances with $nf=7$ (the $\times$ symbol in the graph),
which shows more extreme results as elimination can significantly
speed up or slow down the problem solving.
The slowing down is an interesting discovery of this experiment.
An explanation is that
variable elimination changes the topology of the problem, which may
affect the dynamic variable ordering heuristics of the general solver. It is
well known that the performance of a constraint solver may vary
significantly with a different variable ordering.

To have a better idea on how the elimination algorithm performs, we
look at instances with various hardness. We now zoom into the case
of $nf=8$ and $e=588$ --- this has a large cpu time and is more tight
from the the experiments in Table ~\ref{tbl:random}.
We remark that our algorithm performs similarly in all these cases, so
we just look at the details of a specific one with
the results for all configurations where the tightness changes from
$0.70$ to $0.80$ with a step of $0.01$.
The results are shown in Figure.~\ref{fig:588cpu}
(cpu time) and Figure~\ref{fig:588bks} (the number of backtracks).
Again, the cpu time is the sum of the cost of 20 instances per
parameter setting while the number of backtracks is their average.
When the tightness is $0.79$ and $0.80$ the problem instances become
very simple with less than $200$ backtracks, we do not expect the
elimination algorithm to improve the performance of the constraint
solver although it reduces the number of backtracks. For most non
trivial problems, elimination does help to improve the efficiency
(both cpu time and the number of backtracks) significantly. When
using 10 instances, we also observed that when the tightness is
$0.73$, the elimination leads to a worse performance of the general
problem solver in terms of both cpu time and the number of
backtracks.

\begin{figure}
\includegraphics{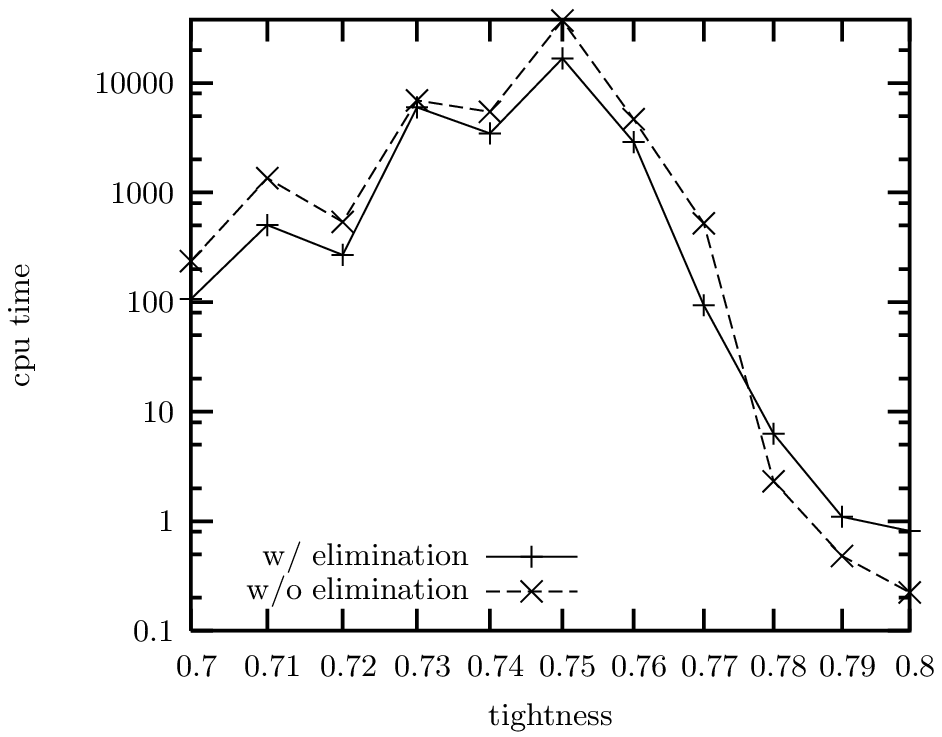}
    \caption{\label{fig:588cpu} The cpu time for instances with $n=d=50$, $nf=8$ and $e=588$ and tightness varying
    from $0.70$ to $0.80$. }
\end{figure}

\begin{figure}
\includegraphics{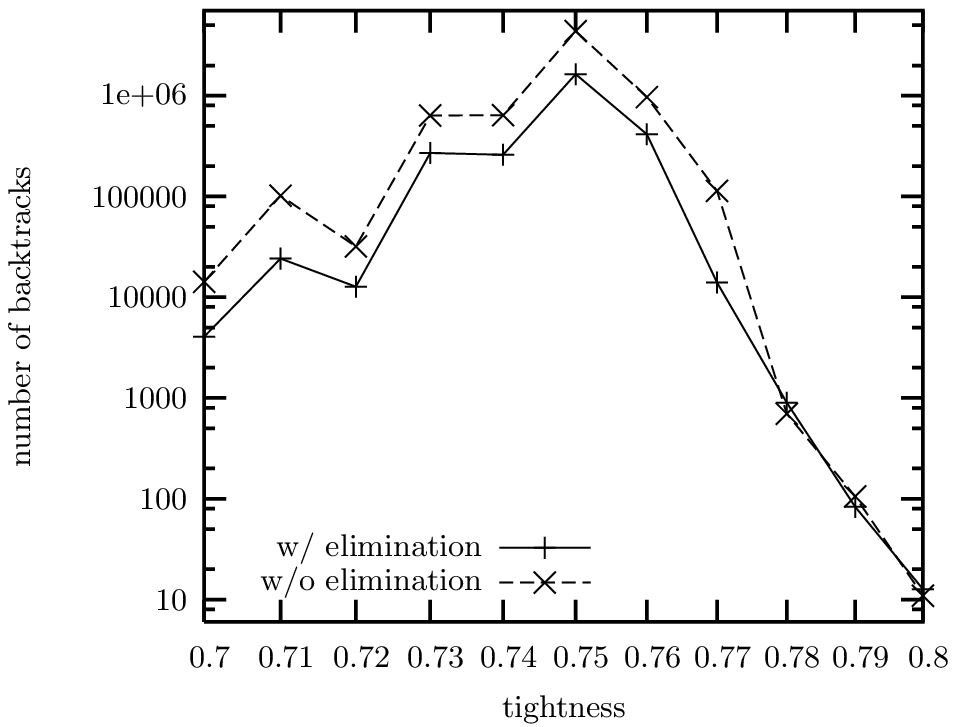}
    \caption{\label{fig:588bks} The number of backtracks for instances with $n=d=50$, $nf=8$ and $e=588$ and tightness varying
    from $0.70$ to $0.80$. }
\end{figure}

As observed, when the number of functional constraints in the random
problems increases, the problem instances become trivial (i.e., very
few backtracks are needed to solve them). This makes it hard to
fully evaluate the elimination algorithm. To reduce the potential
inconsistency caused by the tight functional constraints which
makes the problems easy, we use identity functions
(i.e., $x = y$, a special case of functional
constraint), instead of arbitrary functional constraints.
We remark that we could also have used a permutation form of the
identity function but that would have made the problem instance
construction more complex.
With identity functions, no inconsistency will result directly from
functional constraints. The experimental results given
in Table~\ref{tbl:identity} show that
we can create non trivial problems with much more functional constraints.

\begin{table}
\caption{ The experimental results for random problems with
identity functions (the problem parameters are $n=d=e=100$).}
\begin{center}
\begin{tabular}{rrrrrr} \hline\hline
 &  & \multicolumn{2}{c}{Elimination} &
\multicolumn{2}{c}{No Elimination} \\ 
\raisebox{1.5ex}[0pt]{{\em nf}} & \raisebox{1.5ex}[0pt]{tightness} &
cpu time (s) & \#backtracks & cpu time (s) & \#backtracks \\ \hline
10 & 0.04 & 11.5  &  779.6  &   34.6   & 1870.1 \\ 
20 & 0.04 &  1.4  &  4.5 &  9606.8 &  131724.0 \\ 
30 & 0.08 &  1.1 &  0.2 &  40.7 &  564.6 \\ 
40 & 0.19 & 1.1 &   0 &  177.9 &  1275.7 \\ 
50 & 0.20 &  0.7 &   0 &  893.1  &  64163.2\\ 
60 & 0.25 &  0.6 &   0 &  2.2 &  1.4 \\ \hline\hline
\end{tabular}
\end{center}
\label{tbl:identity}
\end{table}

For this set of instances, the elimination can speed up the problem
solving by up to several orders of magnitude (in terms of both cpu
time and the number of backtracks). An important observation is that
when there is a significant amount of functional constraints, the
number of backtracks needed after elimination can be as small as 0.
However, without elimination, the general problem solving may need
a large number of backtracks (up to five orders of
magnitude larger).

In summary, from our experiments, for non trivial random problems,
elimination can improve the efficiency of a general constraint
solver by several times to several orders of magnitude. We also
observe that the elimination could make a solver slower possibly due
to the change of the topological structure of the problems. However,
the slowdown only occurs rarely in our experiments. \\

\subsection{Search Space Reduction through Elimination}

The experiments show that the elimination could reduce the search
space significantly. In fact, we can show that the elimination can
help reduce the size of the search space (in terms of the current
domains of the variables).

\begin{proposition}
Given a CSP problem $P$, let $P'$ be the problem resulting from
applying the elimination algorithm to $P$. After enforcing arc
consistency on $P'$ and $P$, for each variable of $P'$, its domain
in $P'$ is a subset of that in $P$.
\end{proposition}

\begin{proof}
Instead of  proving the original proposition, we prove the following
claim: given a CSP problem $P=(V, D, C)$ and a constraint $c_{ij}$
functional on $j$,  let $P'=(V-\{j\}, D', C')$ be the problem
resulting from the elimination of $j$. Any value, not from $D_j$,
removed by enforcing arc consistency on $P$ will be removed by
enforcing arc consistency on $P'$. Let $P^1=(V, D^1, C)$ and
$P'^1=(V-\{j\}, D'^1, C')$ be the result of enforcing arc
consistency on $P$ and $P'$ respectively. Equivalently, we will show
that there exists $D''_j \subseteq D_j$ such that $P'' = (V, D'^1
\cup \{D''_j\}, C)$ (i.e., ``plug'' the domains of $P'^1$ to $P$) is
arc consistent. If this claim holds, the proposition holds by
applying the claim repeatedly (as the elimination proceeds).

Let the neighbors of $j$ in $P$ be $i, k_1, \ldots, k_m$. For any
constraint $c_{lk} \in C - \{c_{ij}, c_{jk_1},$ $\ldots, c_{jk_m}\}$,
$c_{lk} \in C'$. Since $c_{lk}$ is arc consistent in $P'^1$,
$c_{lk}$ is arc consistent in $P''$.

We next show $c_{ij} \in C$ is arc consistent with respect to
$D'^1_i$ and $D_j$, i.e., for any value $a \in D'^1_i$, there is a
support in $D_j$. Assume there is no support in $D_j$, by the
definition of substitution, $a$ has no support in any domain of
$k_n$ ($n \in 1..m$), which contradicts that $c_{ik_n} \in C'$ is
arc consistent with respect to $D'^1_i$ and $D'^1_{k_n}$.
Furthermore, let the support of $a$ in $D_j$ be $b$. We claim $b$
has a support with respect to any constraint $c_{jk}$ ($k \in \{k_1,
\ldots, k_m\}$) in $P'^1$. Otherwise, $a$ has no support with respect
to $c_{ik} \in C'$ for some $k \in \{k_1, \ldots, k_m\}$, by the
definition of substitution. It contradicts the fact that $c_{ik}$ is
arc consistent in $P'^1$. Let $D''_j = D_j - \{b ~|~ b \mbox{ has no
support with respect to $c_{ji}$ and } D'^1_i \}$. Clearly $D''_j$
is not empty (because $c_{ij}$ is arc consistent in $P'^1$). It can
be shown that $c_{ji}$ is arc consistent over domains $D^{'1}_i$ and
$D''_j$ and $c_{ij}$ is still arc consistent on these domains too.
In other words, $c_{ij}$ is arc consistent in $P''$.

Similarly, we can show that $c_{jk_1}$, \ldots, and $c_{jk_m}$ are arc
consistent over $D^{'1}_{k_n}$ ($n \in k_1 .. k_m$) and $D''_j$, and
thus in $P''$. 
\end{proof}

Furthermore, after enforcing arc consistency on $P$ and $P'$, for
some variables of $P'$, its domain in $P'$ is a proper subset of
that in $P$. Consider the following example: $V= \{x,y,z\}$, $x,y,z
\in \{1,2\}$, and $c_{xy} = \{(1,1), (2,2)\}$, $c_{yz} = \{(1,2),
(2,1)\}$, $c_{zx} = \{(1,1), (2,2)\}$. This problem is arc
consistent and the domains of the variables are $\{1,2\}$. However,
the problem $P'$ resulting from variable elimination has an empty
domain.

Thus, we see that using the elimination algorithm together with arc
consistency for the non-functional constraints leads to a ``higher
amount'' of consistency.

\section{Beyond Binary Constraints}
\label{sec:beyond}

Our presentation so far is based on binary constraints. To model
real life problems, non-binary constraints are often useful. 
In this section, we discuss the potential extension
of the work reported in this paper. The first subsection is to
generalize substitution to non-binary constraints in extensional
form, and the second proposes an approach to processing non-binary
constraints in intensional form.

\subsection{Variable Elimination and Non-binary Functional
Constraints} \label{sec:non-binarySubstitution}
\newcommand{\newC}[3]{c_{#1 \cup #2 - \{#3\}}}

In this section, we discuss the treatment of the generalization of
binary functional constraints and substitution to non-binary
constraints.
We first generalize the functional property from binary constraints to
non-binary constraints.
Then, we show how a variable is substituted by a
{\em set of variables}.

A non-binary constraint is denoted by $c_{S}$ where $S$ is the set
of variables in the constraint. 
A linear equation
$x+y+z=8$ with finite domains for $\{x, y, z\}$ is a non-binary
constraint. 
We now define variable instantiations.
An instantiation of a set of
variables $Y$ is an assignment of values to the variables in $Y$. It
is usually denoted by a sequence. An instantiation is 
denoted by a character with a bar, for example, $\bar{a}$.

\begin{definition}
A constraint $c_{S}$ is {\em functional on $j(\in S)$} if for any
instantiation $\bar{a}$ of $S-\{j\}$, there is at most one value of
$j$ such that this value and $\bar{a}$ satisfy $c_{S}$. A constraint
$c_S$ is {\em functional} if it is functional on some variable $j
\in S$.
\end{definition}

\begin{eg}
Consider a constraint $x+y+z=8$ with $x,y,z \in \{1,2,3\}$. Let
$\bar{a}=(1,1)$ and $\bar{b}=(2,3)$ be two instantiations of
$(x,y)$. For $\bar{a}$, no value for $z$ can be found to satisfy the
constraint. For $\bar{b}$, value $3$ is the only value for $z$ to
satisfy the constraint. It can be verified that the constraint is
functional on $z$, and similarly on $x$ and on $y$.
\end{eg}

\begin{eg}
The constraint $ x^2 + y^2 + z = 8$ with $x,y,z \in
\{-1, -2, -3, 0, 1, 2, 3\}$ is functional on $z$ but not functional
on $x$ or $y$.
\end{eg}

The idea of variable substitution is applicable to the functional
non-binary constraints defined above. However, we need more
generalised operations to implement variable substitution. Below, we
always take a constraint $c_S$ as a non-binary relation whose tuples
are given explicitly. It is also helpful to recall that a relation
is simply a set, and we can apply set operations like intersection
to relations. For example, the constraint $x+y+z=8$ with $x,y,z \in
\{1,2,3\}$ is taken as $\{(2,3,3), (3,2,3), (3,3,2)\}$ where each
tuple is an instantiation of variables $(x,y,z)$.

In the context of non-binary constraints, the composition of two
constraints $c_S$ and $c_T$ with respect to a variable $i \in S \cap
T$, denoted by ``$\circ_i$,'' is defined below.

\[
  \begin{array}{cl}
   c_S \circ_i c_T = \{\bar{a} ~|~ & \bar{a}\mbox{ is an instantiation of
   $S \cup T - \{i\}$ and there exists $a \in D_i$} \\
   & \mbox{such that $(\bar{a},a)$ satisfies both $c_S$ and $c_T$}.\}
  \end{array}
\]

If $c_S$ is functional on $j$ and $j$ is a variable of constraint
$c_T$, to substitute the variable $j$ in $c_T$ in terms of $c_S$ is
to replace $c_T$ by $c_S \circ_j c_T$.

\begin{definition}
Consider a CSP $(N,D,C)$ and two constraints $c_S$ and $c_T$ in $C$.
Assume $c_S$ is functional on $j \in S \cap T$. To {\em substitute
$j$ in constraint $c_T$ using $c_S$} is to get a new CSP $(N,D,C')$
where $C' = (C - \{c_T\}) \cup \{c'_{S \cup T - \{j\}}\}$ and $c'_{S
\cup T - \{j\}} = c_{S \cup T - \{j\}} \cap (c_S \circ_j c_T)$.
\end{definition}

The variable substitution preserves the solution of a CSP.

\begin{property} \label{prop:nonBinarySubstitution}
Given a CSP $(N,D,C)$, a constraint $c_S \in C$ functional on $j$,
and a constraint $c_T \in C$ where $j \in T$, the new problem
obtained after $j$ in $c_T$ is substituted using $c_S$ is equivalent
to $(N,D,C)$.
\end{property}

\begin{proof}
Let the new problem after $j$ in $c_T$ is substituted be
$(N, D, C')$ where $C' = (C - \{c_T\}) \cup \{c'_{S \cup T
-\{j\}}\}$ and $\newC{S}{T}{j}' = \newC{S}{T}{j} \cap (c_S \circ_j
c_T)$.

Assume $\bar{a}$ is a solution of ($N, D, C$). We shall show that
$\bar{a}$ also satisfies $C'$.

Given a set of variables $Y$, $\bar{a}_Y$ will be used to denote the
values in the solution $\bar{a}$ for the variables in $Y$. $C'$
differs from $C$ in that it has the new constraint
$\newC{S}{T}{j}'$. It is known that $\bar{a}_S \in c_S$, $\bar{a}_T
\in c_T$, and  if there is $\newC{S}{T}{j}$ in $C$, $\bar{a}_{S \cup
T - \{ j\}}$ satisfies $\newC{S}{T}{j}$. The fact that
$\newC{S}{T}{j}' = (c_S \circ_j c_T) \cap \newC{S}{T}{j}$ implies
$\bar{a}_{S \cup T - \{ j\}} \in \newC{S}{T}{j}' (\in C')$. Hence,
$\newC{S}{T}{j}'$ is satisfied by $\bar{a}$.

Conversely, we need to show that any solution $\bar{a}$ of ($N, D,
C'$) is a solution of ($N, D, C$). Given the difference between $C'$
and $C$, it is only necessary to show $\bar{a}$ satisfies $c_T$. To
facilitate the following proof, we write $\bar{a}_S$ as
$(\bar{a}_{S-\{j\}}, \bar{a}_j)$, $\bar{a}_T$ as
$(\bar{a}_{T-\{j\}}, \bar{a}_j)$. We have $(\bar{a}_{S-\{j\}},
\bar{a}_j) \in c_S$ and $\bar{a}_{S \cup T - \{ j\}} \in
\newC{S}{T}{j}'$. Assume, by contradiction, $(\bar{a}_{T-\{j\}},
\bar{a}_j) \notin c_T$.
Since $\bar{a}_{S \cup T - \{ j\}} \in \newC{S}{T}{j}'$, there must
exist $b \in D_{j}$ such that $b \neq \bar{a}_j$, $(\bar{a}_{T
-\{j\}}, b)$ satisfies $c_T$, and $(\bar{a}_{S -\{j\}}, b)$
satisfies $c_S$, contradicting that $c_S$ is functional on $j$. So,
$\bar{a}_T$, that is $(\bar{a}_{T-\{j\}}, \bar{a}_j)$, satisfies
$c_T$. 
\end{proof}

A CSP $(N,D,C)$ with non-binary functional constraints can be
reduced by variable substitution in a similar way as developed in
this paper. In the non-binary case, we note that the complexity of
the algorithm is more expensive due to the composition operation
(which is very close to the {\em join} operation in relational databases).

\subsection {Variable Elimination and Non-binary Constraints}
\label{sec:nonbinary}

Non-binary constraints such as arithmetic or global constraints are
common in CP systems. We discuss how variable elimination of
functional constraints can be applied to these constraints.
Non-binary constraints are either in extensional (defined
explicitly) or intensional (defined implicitly) form. To substitute
a variable in an extensional non-binary constraint, we can follow
the definition given the previous subsection.

In most existing CP systems, for intentional constraints, there are
usually particular propagators with a specific algorithm associated
with them. In this case, the approach using composition is not
directly applicable simply because it has to interact with a
constraint defined in terms of an arbitrary specific propagation
algorithm. We sketch below an approach which allows variable
elimination to be employed with generic propagators. Assume we have
a linear constraint $c_1$: $ax + by + cz < d$ and a constraint
$c_{wy}$ functional on $y$. To substitute $y$ in $c_1$, we simply
modify $c_1$ to be $ax + bw + cz < d$ and mark $w$ as a {\em shadow
variable} ($w$ needs special treatment by the propagator, which will
be clear later). We call $y$ the {\em shadowed} variable. Assume we
also have $c_{uw}$ functional on $w$. To eliminate $w$, $c_1$ is
further changed to $ax + bu + cz < d$. Since $w$ is a shadow
variable, we generate a new constraint $c_{uy}$ using $c_{uw}$ and
$c_{wy}$ in a standard way as discussed in this paper. Now $u$
becomes the shadow variable while the shadowed variable is still $y$
(variable $w$ is gone). Suppose we need to make $c_1$ arc
consistent. First ``synchronize the domains'' of $y$ and $u$ using
$c_{uy}$, i.e., enforce arc consistency on $c_{uy}$. (Note that due
to elimination, $c_{wy}$ and $c_{uw}$ are no longer involved in the
constraint solving.) Next, we enforce arc consistency on $c_1$.
During the process, since $u$ is a shadow variable, all domain
operations are on $y$ instead of $u$. After making $c_1$ arc
consistent, synchronize the domain of $y$ and $u$ again. (If the
domain of $u$ is changed, initiate constraint propagation on
constraints involving $u$.) This approach is rather generic: for any
intensional constraints, synchronize the domains of the shadow
variables and shadowed variables, apply whatever propagation methods
on the shadowed variables (and other non-shadow variables),
synchronize the domains of shadow variables and shadowed variables
again. In fact, the synchronization of the domains of the shadow and
shadowed variables (for example, $u$ and $y$ above) seems be readily
implementable using the concept of views \cite{SchulteT05}.

\section{Related Work}
\label{sec:related}

We now discuss variable substitution in the context of CLP followed
by related work in variable substitution algorithms from other
domains. Finally, the relationship to functional, bi-functional and
other variable elimination approaches in the CSP literature.

\subsection{CLP and Constraint Solving}

Logic Programming and CLP \cite{JM94} systems often make use of
variable substitution and elimination. The classic unification
algorithm discussed below is a good example.

A more complex example is CLP($\cal R$) \cite{JMSY92} which has
constraints on finite trees and arithmetic. Variables in arithmetic
constraints are substituted out using a parametric normal form which
is applied during unification and also when solving arithmetic
constraints. Our approach is compatible with such CLP solvers which
reduce the constraint store to a normal form using variable
substitution. We remark that any CLP language or system which has
finite domain constraints will deal with bi-functional 
constraints simply because of the need
to match an atom in the body with the head of a rule. The question
is how powerful is the approach used. In this paper, we show that a
variable substitution approach is more powerful than just simple
finite domain propagation on equations. The consistency of the CSP
is increased. Our experiments show that the time to solve the
problem can be significantly smaller.

\subsection{Unification, Gaussian Elimination, and Elimination
Algorithm for Functional Constraints}

The algorithm for unification of finite trees, the Gaussian
elimination algorithm for linear constraints and our algorithm for
functional constraints share the same key techniques: variable
substitution and elimination. Such algorithms are also commonly used
in CLP systems. We illustrate this by examples.

The first is the unification of finite trees or terms. Unifying two
terms $f(x,y,z)$ and $f(y, z, g(x))$, where $x,y$ and $z$ are
variables, results in three term equations: $x = y, y = z, z =
g(x)$. These equations can be solved using a variable elimination
method. We select a variable and eliminate it from the system by
substitution. For example, we can select to eliminate $x$ ---
substitute all $x$ by $y$ (using the first equation $x=y$), which
results in $y=z, z=(g(y))$. This process continues until some
constant symbols do not match in an equation, the left hand side
variable appears in a sub-term on the right hand side (or vice
versa), or no new equation can be produced and there is only one
term equation left.

Our second example is equation solving for arithmetic over the real
numbers. A system of binary linear constraints on real numbers can
be solved by the well known Gaussian elimination method. A major
step is to select a variable and eliminate it using an equation.
Specifically, to eliminate $x$ using $ax + by = c$, we substitute
all $x$ in the system by $(c-by)/a$.

Lastly, we look at the case of finite domain constraints considered
in this paper. Given a binary CSP, when we have a variable $x$ and a
general constraint $c_{yx}$ functional on $x$, variable $x$ will be
eliminated by substituting $x$ in the remainder of the constraints.
The substitution here is achieved by (general) composition of
relations. One can show that the substitution of $x$ in Gaussian
elimination produces an equation (a constraint) which is the result
of the composition of the involved equations (constraints). In other
words, given $c_{yx}: x = (c-by)/a$ and $c_{xz}: a_1x + b_1z = c_1$,
the equation $a_1(c-by)/a + b_1z = c_1$ is equal to the composition
of $c_{yx}$ and $c_{xz}$. Therefore, the elimination method proposed
in this paper can be regarded as a generalization of Gaussian
elimination for linear equations to functional constraints defined
over discrete domains.

To see the further similarity among these algorithms, let us examine
the impact of the variable elimination ordering on their efficiency.
As shown in the earlier sections, if all the constraints are known a
priori, we can find a good ordering to make the elimination
algorithm more efficient. The same principle applies to unification
algorithm. Consider the set of term equations $x = f(a,a), y=f(x,x),
z=f(y,y)$. Direct substitution using the ordering of $x,y,z$ is more
expensive than the ordering $z, y, x$. In fact, given a term
equation, it can be unified in linear time by finding a good
variable ordering \cite{PW78}.

In a CLP solver, the constraints are added to the constraint store
dynamically. If the newly added constraint is a binary linear
equation, it has been observed that one can improve the efficiency
by choosing properly a variable to eliminate from the two involved
in the equation \cite{BSTY95}. For example, a brand new variable is
preferred to an old one (a variable occurring previously in the
constraint store). With the new variable, no substitution is
necessary. For elimination using bi-functional constraints, the one
involved in a lesser number of earlier constraints will be
eliminated \cite{ZY02a}. In the case of unification, when there are
several variables that can be eliminated, we choose the one that is
involved in less number of constraints too (for example, \cite{EG88}). The
variable selection idea, together with the disjoint set data
structure and union-find algorithm, had led to almost linear
algorithms in \cite{EG88,ZY02a}.

\subsection{Functional Constraints and Variable Elimination in CSP}
\label{sec:relatedWork}

We now discuss other work related to functional constraints from a
CSP perspective.

Bi-functional constraints, a special case of functional constraints,
have been studied in the context of arc consistency (AC) algorithms
since Van Hentenryck et al. \cite{vHDT92} proposed 
A worst case optimal AC algorithm with $O(ed)$
was proposed in \cite{vHDT92}.
(In many of the papers, bi-functional
constraints were called functional constraints). The special
properties of bi-functional constraints were used to obtain the time
complexity better than that of the optimal AC algorithms such as
AC2001/3.1 ($\bigO{(ed^2)}$) \cite{BRYZ05} for arbitrary binary
constraints. 
A fast AC algorithm for a
special class of {\em increasing} bi-functional constraints
was also proposed in \cite{Liu95}. 
Here, our elimination algorithm solves the consistency of functional
constraints and variable substitution incorporates the effect of the
functional part of the problem into the rest of the non-functional
constraints. Thus, it gives a higher level of consistency as it
achieves global consistency for the functional constraints rather
than local consistency like arc consistency. At the same time, it
may simplify the remainder of the constraints, thus reducing the
problem further.

A new type of consistency, label-arc consistency, was introduced
in \cite{AB96} and they showed that
bi-functional constraints with limited extensions to other
constraints can be (globally) solved, but no detailed analysis of
their algorithms is given. In \cite{ZYJ99}, we proposed a variable
elimination method to solve {\em bi-functional} constraints in
$\bigO(ed)$. Bi-functional constraints also belong to the class of
``Zero/One/All'' constraints which was shown to be one of the earliest
classes of tractable constraints \cite{CCJ94}. The subclass of
``One'' constraints in the ``Zero/One/All'' class corresponds to
bi-functional constraints. What was not realized in
\cite{CCJ94,ZYJ99} was that because the concern was the tractability
of the class of Zero/One/All constraints, the importance of variable
substitution and class of functional constraints was missed. We also
point out that all the papers above deal with the special case of
bi-functional constraints rather than functional constraints.

David introduced {\em pivot consistency} for binary functional
constraints in \cite{Dav95}. Both pivot consistency and variable
substitution are different ways of reducing a CSP into a special
form. However, there are some important differences between pivot
consistency and variable substitution in this paper. Firstly, the
concept of pivot consistency, a special type of directional path
consistency, is quite complex. It is defined in terms of a variable
ordering, path (of length 2) consistency, and concepts in directed
graphs. As we show in this paper, Variable substitution is a much
simpler concept. It is intuitive and simple for binary CSPs, and it
extends also simply and naturally to non-binary CSPs. Secondly, by
the definition of pivot consistency, to make a CSP pivot consistent,
there must be a certain functional constraint on each of the {\em
non-root} variables. Variable substitution is more flexible. It can
be applied whenever there is a functional constraint in a problem.
Finally, to reduce a problem, the variable elimination algorithm
takes $\bigO(ed^2)$ while pivot consistency algorithm takes
$\bigO((n^2-r^2)d^2)$, where $r$ is the number of {\em root}
variables.

Another related approach is bucket elimination \cite{Dec99}. The
idea in common behind bucket elimination and variable substitution
is to exclude the impact of a variable on the whole problem.  The
difference between them lies in the way variable elimination is
performed. In each elimination step, substitution does not increase
the arity of the constraints while bucket elimination could generate
constraints with {\em higher arity} (possibly with exponential space
complexity). The former may generate more constraints than the
latter, but it will {\em not} increase the total number of
constraints in the problem.

Another methodologically related work is bucket elimination \cite{Dec99}. 
The common idea behind bucket elimination
and variable substitution is to exclude the impact of a variable on
the whole problem. However, they also differ in many aspects. Bucket
elimination deals with general constraints while variable
substitution is applicable only to functional constraints. Bucket
elimination assumes a variable ordering and eliminates the impact of
a variable $j$ on {\em all} relevant constraints that involve
variables before $j$. In contrast, variable substitution can be used
to eliminate the impact of a variable on {\em any number} of
relevant constraints. The ways to eliminate a variable are different
for the two methods. For example, consider the CSP shown in
Figure~\ref{fig:width}(a) where $c_{12}, c_{23}, c_{34}$, and
$c_{45}$ are functional.

\begin{figure}
\includegraphics{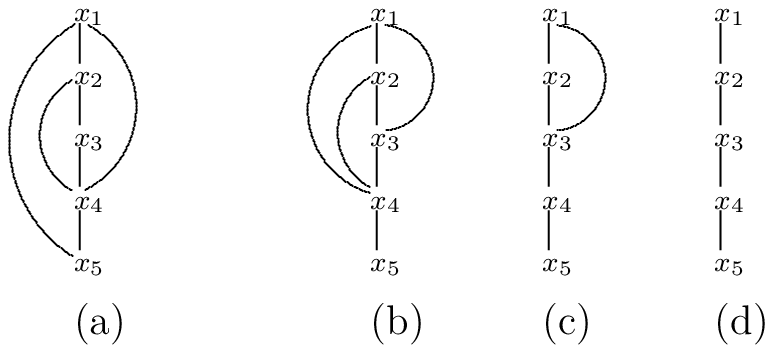}
\caption{\label{fig:width} A CSP with variable ordering $x_1, x_2,
\cdots, x_5$.}
\end{figure}

Assume there is a variable ordering $x_1, x_2, \cdots, x_5$. The
variables will be eliminated in the reverse of the variable
ordering. When eliminating a variable, say variable $x_4$, bucket
elimination considers constraints only involving variables before
(including) $x_4$ and ignores other constraints. In this example,
$c_{45}$ is ignored while constraints $c_{14},c_{24}, c_{34}$ are
considered relevant. After eliminating $x_4$, the new {\em ternary}
constraint $c_{\{1,2,3\}}$ on $x_1,x_2,x_3$ is added to the problem.

In the variable substitution method, for variable $x_4$ and the
constraint $c_{34}$ functional on it, we can choose to substitute
$x_4$ in one or some of the constraints $c_{41}, c_{42}$ and
$c_{45}$, depending on a specific setting (for example, a static CSP or
incremental CSP). If we choose to substitute $x_4$ in all these
constraints, new binary constraints $c_{31}, c_{32}, c_{35}$ are
added and old constraints $c_{41}, c_{42}$ and $c_{45}$ are
discarded.

This example shows that in each elimination step, bucket elimination
generates constraints with higher arity than variable substitution
while the latter generates more constraints than the former.
However, the variable substitution method will not increase the
total number of constraints in the problem (because every time a new
constraint is added, an old one is discarded). In the case of
bi-functional constraints, it decreases the total number of
constraints to $n$ after all variables are eliminated.

\section{Conclusion}
\label{sec:conclusion}

We have introduced a variable substitution method to reduce a
problem with both functional and non-functional constraints.
Compared with the previous work on bi-functional and functional
constraints, the new method is not only conceptually simple and
intuitive but also reflects the fundamental property of functional
constraints.

For a binary CSP with both functional and non-functional
constraints, an algorithm is presented to transform it into a
canonical functional form in $\bigO(ed^2)$. This leads to a
substantial simplification of the CSP with respect to the functional
constraints. In some cases, as one of our results (Corollary 2)
shows, the CSP is already solved. Otherwise, the canonical form can
be solved by ignoring the eliminated variables. For example, this
means that search only needs to solve a smaller problem than the one
before variable substitution (or elimination).

Our experiments show that variable elimination can significantly (in
some cases up to several orders of magnitude) improve the
performance of a general solver in dealing with functional
constraints. Our experiments also show some evidence that although
rarely, the elimination could slow down the general solver in a non
trivial way.

\section*{Acknowledgments}
We thank Chendong Li for helping implement the elimination algorithm
and carrying out some experiments in the earlier stage of this
research, and Satyanarayana Marisetti for writing the code for
generating random functional constraints and the functional
elimination ordering. Portions of this work was supported by National
Univ. of Singapore, grant 252-000-303-112.
\bibliographystyle{acmtrans}
\bibliography{cp}
\end{document}